\documentclass{article}

\PassOptionsToPackage{numbers}{natbib}

\usepackage[preprint]{neurips_2026}

\usepackage[utf8]{inputenc}
\usepackage[T1]{fontenc}
\usepackage{hyperref}
\usepackage{url}
\usepackage{booktabs}
\usepackage{amsmath,amssymb,amsfonts,amsthm,mathtools}
\usepackage{enumitem}
\usepackage{microtype}
\usepackage{xcolor}
\usepackage{array}
\usepackage{nicefrac}

\newtheorem{theorem}{Theorem}[section]
\newtheorem{lemma}[theorem]{Lemma}

\newtheorem{fact}[theorem]{Fact}
\newtheorem{definition}[theorem]{Definition}

\newcommand{\ACzero}{\mathrm{AC}^0}
\newcommand{\TCzero}{\mathrm{TC}^0}
\newcommand{\PCR}{\mathrm{PCR}}

\newcommand{\zo}{\{0,1\}}

\newcommand{\YES}{\mathsf{YES}}
\newcommand{\NO}{\mathsf{NO}}
\newcommand{\MARK}{\mathsf{MARK}}
\newcommand{\BLANK}{\mathsf{BLANK}}
\newcommand{\BITZERO}{\mathsf{0}}
\newcommand{\BITONE}{\mathsf{1}}

\newcommand{\DUMMY}{\mathsf{DUMMY}}

\title{Provably Shorter Scratchpads in Hybrid DeltaNet-Attention Decoders}
\author{Tomasz Steifer\\ 
Centre for Credible AI\\ Warsaw University of Technology\\
Rektorska 4, Warszawa 00-614, Poland}

\begin{document}
\maketitle

\begin{abstract}
We investigate the expressive power of hybrid recurrent-attention decoders, a class of architectures used in recent open-source language models such as Qwen3-Next and its successors. These models combine Gated Attention heads with recurrent Gated DeltaNet heads. Is there a formal advantage, in terms of model expressivity or efficiency, to such a hybrid architecture? We show that there is. We define parity-conditioned retrieval task and show that under constant-precision assumption, a Qwen-style hybrid of Gated DeltaNet and Gated Attention solves this task with a constant scratchpad, or equivalently $O(1)$ chain-of-thought steps. In contrast, no similar solution exists for pure Gated DeltaNet models, while pure Gated Attention requires at least a polynomial scratchpad.
\end{abstract}

\section{Introduction}

For almost a decade transformers have been a cornerstone of modern language models but recent years brought a surge of interest in recurrent mechanisms such as in Mamba or DeltaNet. Both approaches to language processing have their pros and cons. For instance, attention can track relations between distant tokens with little cost but has a growing key-value cache. Moreover, standard transformers can struggle with simple formal tasks such as deciding whether the number of ones in the prompt is even or odd, unless allowed to produce large amounts of chain-of-thought tokens. Recurrent scans such as DeltaNet maintain a fixed-size state, which is efficient for online state tracking but cannot retain an arbitrarily large memory. This suggests hybrid architectures as a way to combine the best of the two worlds. Qwen3-Next is a motivating example: its public model card describes a 48-layer hybrid layout that alternates runs of Gated DeltaNet blocks with Gated Attention blocks~\citep{qwen3next_modelcard}.  

This paper contributes to the study of the expressivity of such hybrid models. Our main question is following: are hybrid models more expressive? If yes, then under what assumptions? To answer that we give a natural example of a formal task which is easy for a Qwen-inspired hybrid architecture but is hard for decoders that use only Gated Attention or Gated DeltaNet layers.

The task in question is as follows. The model reads a bit string \(Y=(Y_1,\ldots,Y_n)\), followed by a sequence encoding an index \(j\).  It must output
\[
  Y_j\oplus(Y_1\oplus\cdots\oplus Y_n).
\]
This is a toy example of a more general template. The language model is given some kind of text (e.g. code, server logs, chat context) and a query conditioned on some global property of the text. For instance, an agent is asked to list all the events that occurred while the user was logged in.
Informally speaking, the computation has two independent subproblems:
\begin{enumerate}
  \item a serial scan subtask, namely the parity of the input sequence; and
  \item a delayed retrieval subtask, namely addressed retrieval of the queried bit \(Y_j\) after the query is revealed.
\end{enumerate}
A recurrent DeltaNet head can solve the first subproblem at constant precision, and an attention head can solve the second using positional encodings.  Either component alone is forced into a tradeoff: pure attention needs scratchpad (chain-of-thought tokens) to compute parity, while a pure DeltaNet must compress the input sequence into its rounded state before it sees the query.

\paragraph{Contribution.}
We formalize the Qwen-style hybrid decoder under a practically motivated constant-precision assumption.
\begin{enumerate}[leftmargin=2em]
  \item We define Parity-Conditioned Retrieval (\(\PCR\)), a toy task that combines online parity tracking with delayed queried retrieval.

  \item We give an explicit construction of a hybrid decoder (with constant depth and logarithmic model dimension) that solves the task: a two-coordinate positive-gate Gated DeltaNet row tracks parity, and two finite-precision attention heads retrieve the query address and then the selected bit. The hybrid solution works in a single iteration---no scratchpad / chain-of-thought generation.
  \item We observe that two matching lower bounds hold for decoders limited to only one kind of heads. Under constant precision, Pure Gated Attention cannot compute parity unless it is allowed to generate a scratchpad of a size polynomial in the input length $N$. Pure Gated DeltaNet with \(o(N)\) bits of rounded recurrent state fails, even if it may generate arbitrarily many scratch tokens after the prompt.
\end{enumerate}

\paragraph{Precision convention}
The exact statement of our results depends on the adopted precision convention. As a side-note, it seems that a similar separation should hold for a weaker convention (rounding after each operation), perhaps at the cost of a faster scaling model dimension. This conjecture is not verified or formalized in the current paper. See section 3.2. for more details on this convention.

\paragraph{AI tools declaration}
ChatGPT (versions 5.4-5.5) was used during proof exploration to suggest candidate proof strategies, check informal derivations, and help identify issues in failed proof attempts. The authors independently verified all formal statements and proofs. ChatGPT and Deepseek were also used for drafting and grammar/style editing. To avoid any conflict of interest, no Qwen models were involved.

% \paragraph{Scope.}
% The theorem is a separation for an idealized constant-precision model, not a claim about trained model performance.  The attention gadget below uses finite-precision underflow to obtain exact one-hot selection.  Thus the upper bound should be read as a clean expressivity construction for finite-precision attention, not as a statement that literal production kernels implement this exact arithmetic.

\section{Related work}
\label{sec:related-work}

\subsection{Hybrid attention--recurrent decoders}

The closest contemporary work is Olmo Hybrid~\citep{merrill2026olmo_hybrid}, whose authors ask the same high-level question as we do: whether mixing attention with recurrent linear-time layers yields a more expressive model than simply inheriting the strengths of each component. The main contribution of this paper is in training a 7B Olmo Hybrid model in which most attention layers are replaced by Gated DeltaNet. They also give a formal separation between the hybrid architecture and pure decoders with only one kind of head.

Our theoretical contribution has similar shape to that in Olmo Hybrid paper. Attention provides accurate queried retrieval, while the recurrent component supplies online state tracking. Olmo Hybrid studies this complementarity under logarithmic precision average-hard-attention idealization and with no chain-of-thought. In contrast, our separation is given for constant-precision decoder and applies even if models can use logarithmic scratchpads. Furthermore, our lower bounds are unconditional, while the Gated Attention lower bound in Olmo Hybrid paper crucially uses an unproven complexity-theoretic conjecture.

There is also an architectural difference.  Olmo Hybrid interleaves Gated DeltaNet with multihead attention.  Our motivating implementation is closer to Qwen3-Next, which combines Gated DeltaNet with Gated Attention, with additional QK normalization~\citep{qwen3next_modelcard}. While our lower bound applies to transformer models with standard attention, our upper bound uses a restricted QK-normalized Gated Attention instead. 

\subsection{Transformer expressivity}

Transformer expressivity has a large formal-language and circuit-complexity literature; see Strobl et al.~\citep{strobl2024survey} for a survey.  Existing results are sensitive to modeling choices such as positional encodings, hard versus soft attention, number of layers and heads, causal masking, layer normalization, finite versus infinite precision, encoder--decoder versus decoder-only architecture, and whether the model may generate chain-of-thought tokens.  Early hard-attention lower bounds put unique hard-attention transformers in \(\ACzero\) and therefore rule out parity~\citep{hahn2020theoretical,hao2022formal}.  Average or saturated hard-attention variants are stronger but still admit \(\TCzero\)-type circuit upper bounds~\citep{merrill2022saturated}.  Other works give automata and logical characterizations of hard-attention variants, including star-free languages and first-order definability~\citep{angluin2023masked,barcelo2023logical}.

For softmax attention, the picture depends strongly on precision and positional encodings.  Log-precision transformers admit parallel circuit simulations~\citep{merrill2023parallelism}, while C-RASP and related work characterize counting abilities of softmax attention~\citep{yang2024counting}.  Recent parity constructions show that bounded-depth softmax transformers can compute parity under sufficiently fast growing positional-encoding; in particular, Kozachinskiy et al.~\citep{kozachinskiy2026parity} give a softmax construction with length-independent polynomial positional encodings and no layer normalization, together with a one-layer lower bound.  With generated chain of thought, transformers become much stronger: polynomial-length CoT can simulate broad classes of computation~\citep{li2024cot,merrill2024cot}, and Jiang et al.~\citep{jiang2025softmax_tc} recently proved Turing-completeness for length-generalizable softmax CoT transformers via CoT C-RASP. 

% Our pure gated-attention lower bound is deliberately in the strict-constant-precision, one-step-\(\ACzero\) regime.  It is therefore compatible with positive results for higher precision, stronger positional features, more softmax-realistic models, or polynomial generated CoT.  What it shows is that when a pure gated-attention next-token map remains \(\ACzero\)-realizable, the serial parity component must be paid for in scratch length.

\subsection{Recurrent neural-network expressivity}

Classical work treats recurrent networks as sequential machines.  With idealized continuous states, recurrent neural nets can simulate Turing machines~\citep{siegelmann1995computational}; finite-state and saturated viewpoints connect recurrent networks to finite automata and regular languages~\citep{siegelmann1996finite}.  Recent work on recurrent neural language models refines this viewpoint by studying distributions over strings, connecting RNN language models to probabilistic finite-state automata and proving lower bounds and precision requirements~\citep{svete2024lower}.

Modern linear RNNs and state-space-style models revive these questions under hardware-efficient recurrence.  Gated DeltaNet combines a delta-rule update with gates for memory control~\citep{yang2025gdn}.  Expressivity papers on DeltaNet and related linear RNNs emphasize state tracking.  Grazzi et al.~\citep{grazzi2025negative} show that finite-precision linear RNNs with only positive state-transition eigenvalues cannot solve parity, that non-triangular transitions are needed for modulo-three counting, and that products of identity-minus-outer-product matrices can learn any regular language.  Siems et al.~\citep{siems2025deltaproduct} propose DeltaProduct, generalizing DeltaNet through products of Householder updates, and give a finite-precision characterization of how this mechanism improves state tracking as the number of updates per token increases. 

% Our upper bound uses the recurrent component in the simplest possible way---as a finite-state parity tracker---while our pure-GDN lower bound isolates the complementary failure mode: a fixed rounded recurrent state cannot retain arbitrary delayed-addressed table information.

\section{Preliminaries}
\subsection{Decoders}
Fix a finite vocabulary $\Sigma$, a fixed scratch alphabet $\Gamma\subseteq\Sigma$, and two answer tokens $\YES,\NO\in\Sigma$.  

% A length-indexed hybrid decoder family is a sequence of deterministic decoders $\mathcal D=\{D_M\}_{M\ge1}$, where $M$ is the maximum context length available to the decoder.  The vocabulary $\Sigma$, the precision parameter $s$, the number of layers $L$, the number of heads per layer, and the layer-type sequence
% \[
%   \tau_1,\ldots,\tau_L\in\{\mathsf{GA},\mathsf{GDN}\}
% \]
% are independent of $M$.  The hidden dimension $d_M$ may depend on $M$.
The decoder receives a finite sequence $z_1,\ldots,z_N$ of tokens from $\Sigma$ as input prompt. 
As the first step the decoder applies a fixed embedding map 
\[
  \operatorname{Emb}_M:\Sigma\times[M]\to\mathbb F_s^{d_M}.
\]
On an input $z_1,\ldots,z_m\in\Sigma$, we get
\[
  h_i^{(0)}=\operatorname{Emb}_M(z_i,i),\qquad 1\le i\le m.
\]
The sequence $h_i^{0}$ goes through a sequence of $l$ layers of types $\tau_1,\ldots,\tau_L\in\{\mathsf{GA},\mathsf{GDN}\}$. The $j$-th layers receives the sequence of vectors $h_i^{j-1}$ as input and returns a sequence of vector $h_i^{j}$.
Each of these layers consists of a token-mixing block followed by a feed forward network applied independently to each position. 
Each token-mixing block is applied on a residual stream. If \(T_i^\ell\) is the
token-mixing output at position \(i\), then
\[
  \bar h_i^\ell=[h_i^{\ell-1}+T_i^\ell]_s,
\]
and the following position-wise MLP is also residual:
\[
  h_i^\ell=[\bar h_i^\ell+\operatorname{MLP}^\ell(\bar h_i^\ell)]_s.
\]
There are two types of token-mixing blocks:

\paragraph{Gated Attention.}
 A Gated Attention (or $\mathrm{GA}$) consists of $H$ attention heads.  For each head
$a\in[H]$, let $d_h$ denote the
head dimension. For simplicity we assume that $d_M= H\cdot d_h$.  The $a$-th head of $l$-th layer is specified by affine maps
\[
Q^a,K^a:\mathbb R^d\to\mathbb R^{d_h},
\qquad
V^a:\mathbb R^d\to\mathbb R^{d_h},
\qquad
G^a:\mathbb R^d\to\mathbb R^{d_h}.
\]
At every position $i$, define
\[
q_i^a=\operatorname{RMSNorm}_s(Q^a(h_i^{\ell-1})),
\qquad
k_i^a=\operatorname{RMSNorm}_s(K^a(h_i^{\ell-1})),
\qquad
v_i^a=V^a(h_i^{\ell-1}).
\]
and define the output gate
\[
\gamma_i^a = \sigma(G^a(h_i^{l-1}))\in(0,1)^{d_h},
\]
where $\sigma$ is a positive sigmoid gate applied coordinate-wise. In our upper bound construction, this gate is effectively absorbed into output projections.

For $j\le i$, define the causal attention score
\[
s_{ij}^a
=
\frac{\langle q_i^a,k_j^a\rangle}{\sqrt{d_h}}.
\]
We assume that the positions $j>i$ are causally masked for each querying token $i$.  The attention weights are
\[
A_{ij}^a
=
\frac{\exp(s_{ij}^a)}
{\sum_{r=1}^{i}\exp(s_{ir}^a)}
\qquad\text{for }j\le i,
\]
and $A_{ij}^a=0$ for $j>i$.

The attention output of head $a$ at position $i$ is
\[
u_i^a
=
\sum_{j=1}^{i} A_{ij}^a v_j^a
\in\mathbb R^{d_h}.
\]
The gated head output is
\[
o_i^a
=
\gamma_i^a\odot u_i^a.
\]

Finally, the outputs of all heads are concatenated and projected back to
the model dimension:
\[
z_i
=
W_O
\begin{bmatrix}
o_i^1\\
\vdots\\
o_i^{H}
\end{bmatrix}
+
b_O
\in\mathbb R^d .
\]

\paragraph{Gated DeltaNet.}
A Gated DeltaNet (or $\mathsf{GDN}$) layer consists of 
$H$ heads.
For each head
$a\in[H]$, let $d_h$ denote the
head dimension. For simplicity, we assume that $d_M= H\cdot d_h$.  The $a$-th head of $l$-th layer is specified by affine maps
\[
Q^a,K^a:\mathbb R^d\to\mathbb R^{d_h},
\qquad
V^a:\mathbb R^d\to\mathbb R^{d_h},
\]
\[
A^a,B^a:\mathbb R^d\to\mathbb R,
\qquad
G^a:\mathbb R^d\to\mathbb R^{d_h}.
\]
At position $i$, the head computes:
\[
q_i^a = \frac{Q^a(h_i^{l-1})}{\|Q^a(h_i^{l-1})\|_2},\qquad
k_i^a = \frac{K^a(h_i^{l-1})}{\|K^a(h_i^{l-1})\|_2},\qquad
v_i^a = V^a(h_i^{l-1}),
\]
and
\[
\alpha_i^a=\phi_\alpha^a(A^a(h_i^{l-1})),
\qquad
\beta_i^a=\phi_\beta^a(B^a(h_i^{l-1})),
\]
where \(\phi_\alpha,\phi_\beta:\mathbb R\to[0,1]\) are fixed nonlinear activation functions.
It also computes
\[
\gamma_i^a = \sigma(G^a(h_i^{l-1}))\in(0,1)^{d_h},
\]
where $\sigma$ is applied coordinate-wise.
The head scans the input sequence from left to right, maintaining a recurrent
state \(S_i^a\in F_s^{d_h\times d_h}\). The state is initialized at a fixed
\(S_0^a\) and updated as follows.
\[
S_i^a
=
\alpha_i^a
\Bigl(S_{i-1}^a
-
\beta_i^a (S_{i-1}^a k_i^a)(k_i^a)^\top
\Bigr)
+
\beta_i^a v_i^a(k_i^a)^\top .
\]

The head reads from the updated state using the normalized query:
\[
u_i^a = S_i^a q_i^a \in\mathbb R^{d_h}.
\]
The gated head output is
\[
o_i^a = \gamma_i^a\odot \operatorname{RMSNorm}(u_i^a) .
\]

Finally, the outputs of all recurrent heads are concatenated and
projected back to the model dimension:
\[
z_i
=
W_O
\begin{bmatrix}
o_i^1\\
\vdots\\
o_i^{H}
\end{bmatrix}
+
b_O
\in\mathbb R^d .
\]

% The sequence $(z_1,\ldots,z_M)$ is the output of the Gated DeltaNet
% layer.

% A Gated DeltaNet (or $\mathsf{GDN}$) layer consists of 
% $H_{\rm GDN}$ heads. Let $d_k^a,d_v^a$ be fixed head
% dimensions.  The $a$-th head maintains a causal state
% \[
% S_i^a\in\mathbb F_s^{d_v^a\times d_k^a},
% \qquad S_0^a\in\mathbb F_s^{d_v^a\times d_k^a},
% \]
% where $S_0^a$ is a fixed initial state.  At position
% $i$, a fixed rounded pointwise parameter map applied to
% $h_i^{(\ell-1)}$ produces
% \[
% q_i^a,k_i^a\in\mathbb F_s^{d_k^a},\qquad
% v_i^a\in\mathbb F_s^{d_v^a},\qquad
% \alpha_i^a,\beta_i^a\in\mathbb F_s\cap[0,1].
% \]
% The state update is
% \[
% S_i^a
% =
% \left[
% S_{i-1}^a
% \bigl(\alpha_i^a(I-\beta_i^a k_i^a(k_i^a)^\top)\bigr)
% +
% \beta_i^a v_i^a(k_i^a)^\top
% \right]_s .
% \]
% The head output is
% \[
% o_i^a=[S_i^a q_i^a]_s\in\mathbb F_s^{d_v^a}.
% \]
% The layer output is the output projection of the
% concatenated head outputs,
% \[
% h_i^{l}
% =
% \left[
% W_O^\ell(o_i^1,\ldots,o_i^{H_{\rm GDN}})
% +
% b_O^\ell
% \right]_s.
% \]
As in a $\mathsf{GDN}$ block, this token-mixing output is followed
by the layer's rounded position-wise MLP with ReLU activation.

\paragraph{Output}
After the final layer, a rounded affine map applied to the last position produces logits in $\mathbb F_s^{|\Sigma|}$.  The next token is chosen by deterministic greedy argmax with a fixed tie-breaking order on $\Sigma$.  Since the logits are rounded and $\Sigma$ is finite, the argmax is a finite lookup on rounded values.

\paragraph{Problem solving and scratchpads}
We say that a decoder $D$ solves a promise problem with scratch budget $S(N)$ if, on every promised prompt of length $N$, greedy decoding emits at most $S(N)$ tokens from the fixed scratch alphabet $\Gamma$ before its first token in $\{\YES,\NO\}$, and that first answer token is correct.

\paragraph{Hybrid and pure decoders}
A pure gated-attention decoder is the special case in which every layer has type $\mathsf{GA}$.  A pure Gated DeltaNet decoder is the special case in which every layer has type $\mathsf{GDN}$. A hybrid decoder contains both kinds of layers stacked in arbitrary order.

% \begin{remark}[Qwen-style layer layouts]
% \label{rem:qwen-layout}
% The definition permits any fixed layer-type sequence.  This includes alternating $\mathsf{GDN}$ and $\mathsf{GA}$ layers, and also Qwen3-Next-style patterns with several Gated DeltaNet blocks for each Gated Attention block.  The upper-bound construction only uses a constant number of effective layers; unused layers can be set to identities by rounded parameters.
% \end{remark}

\subsection{Constant precision}
Our results assume the following practically-motivated constant-precision model. All tensors that
cross module boundaries---hidden states, gates, keys, queries, values, attention
probabilities, recurrent states, residual streams, MLP activations, and logits---
are stored in a fixed finite format \(\mathbb F_s\).  Ordinary scalar operations
are rounded back to this format.  The only exception is that certain reductions
internal to an operator, namely RMS reductions, QK dot products, and softmax
denominators, may use a wider temporary accumulator. Such an accumulator is
not persistent model state and is not exposed to later layers; only its rounded
output is used by subsequent computations.  We note that our convention slightly differs from that of some papers (e.g. \cite{li2024cot}) where all arithmetic is rounded operation by operation. We specify the convention in detail in the Appendix A.

% Fix an integer $s\ge 2$.  Let
% \[
%   \mathbb F_s=\{0\}\cup\{\pm k2^{-s}:1\le k\le 2^{2s}-1\}.
% \]
% Let $B_s=2^s-2^{-s}$ be the largest positive element of $\mathbb F_s$.  The rounding map $[\cdot]_s:\mathbb R\to\mathbb F_s$ sends a real number to the closest element of $\mathbb F_s$, breaking ties by smaller absolute value and saturating to $\pm B_s$ on overflow.  Vectors and matrices are rounded coordinatewise.
% \end{definition}

% For $a,b\in\mathbb F_s$ define rounded binary arithmetic by
% \[
%   a\oplus_s b=[a+b]_s,\qquad
%   a\ominus_s b=[a-b]_s,\qquad
%   a\otimes_s b=[ab]_s.
% \]
% For $x=(x_1,\ldots,x_m)\in\mathbb F_s^m$, let
% \[
%   \operatorname{sum}_s(x_1,\ldots,x_m)
%   =(((x_1\oplus_s x_2)\oplus_s x_3)\cdots)\oplus_s x_m,
% \]
% with the evident convention for $m=1$.  The finite-precision inner product is
% \[
%   \langle x,y\rangle_s=\operatorname{sum}_s(x_1\otimes_s y_1,\ldots,x_m\otimes_s y_m).
% \]
% Constant-precision softmax is computed by rounding the exponentials, summing them with rounded summation, and rounding the divisions:
% \[
%   \operatorname{softmax}_s(z)_i
%   =\left[\frac{[\exp(z_i)]_s}{\operatorname{sum}_s([\exp(z_1)]_s,\ldots,[\exp(z_m)]_s)}\right]_s,
% \]
% whenever the denominator is nonzero. 

\subsection{Circuit complexity}
\begin{definition}
Let $\mathcal{L}_n$ be a finite language of binary words of length $n$. We will say that a circuit $C$ with one output node recognizes $\mathcal{L}_n$ if for each $w$ with $|w|=n$ we have $C(w)=1$ if and only if $w\in \mathcal{L}_n$. Moreover, a language $\mathcal{L}$ is recognized by a family of circuits $C_1,C_2,\ldots$ if $C_n$ recognizes the subset of $\mathcal{L}$ consisting of strings of length $n$.
\end{definition}
\begin{definition}
A language \(\mathcal L\) is in \(\ACzero\) if it is recognized by a family of circuits of polynomial size and constant depth
with unbounded-fan-in AND and OR gates and fan-in-one NOT gates.
\end{definition}
Boolean circuits with unbounded-fan-in AND and OR gates and fan-in-one NOT gates will be simply called circuits for the rest of the paper.
For our lower bound on pure Gated Attention decoders we will use the following standard result from circuit complexity.
\begin{fact}[H\aa stad parity lower bound]
\label{fact:hastad}
For every fixed depth $d\ge2$, any depth-$d$ circuits computing parity on $n$ bits has size at least
\[
  2^{\Omega_d(n^{1/(d-1)})}.
\]
The same lower bound holds for the complement of parity~\cite{furst1984,hastad1986}. In particular, parity is not in $AC^0$.
\end{fact}
For languages over finite non-binary alphabets, we use any fixed binary encoding of the alphabet; this changes circuit size only by constant factors.
% The following upper bound for constant-precision standard transformers is known.
% \begin{fact}[Li-Liu-Zhou-Ma \cite{li2024cot}]
% \label{fact:llzm-log-cot}
% Constant-depth decoder-only transformers with constant-size
% precision, polynomial embedding dimension, and \(O(\log n)\) scratchpad recognize only languages in \(\ACzero\). 
% \end{fact}

% \begin{definition}[$\ACzero$-realizable pure gated-attention decoder]
% A deterministic pure gated-attention decoder with finite scratch alphabet $\Gamma$ is $(d_0,c_0)$-$\ACzero$-realizable if, for every context length $M$ and every possible next token $a\in\Gamma\cup\{\YES,\NO\}$, the predicate
% \[
%   F_M(z_1,\ldots,z_M)=a
% \]
% including the rounded forward pass, rounded logits, deterministic tie-broken argmax, and any length-indexed positional embeddings hardwired at length $M$, is computed by a nonuniform $\ACzero$ circuit of depth at most $d_0$ and size at most $M^{c_0}$.  The input variables of the circuit are only the token identities in the context.
% \end{definition}

% \begin{fact}[Constant-precision transformers are in $\ACzero$]
% \label{fact:llzm}
% Fixed-depth decoder-only transformers with a fixed finite vocabulary, polynomial embedding dimension, length-indexed positional embeddings, constant precision, and no generated chain of thought are $\ACzero$-realizable in the sense above.  This is the upper-bound regime of Li--Liu--Zhou--Ma~\cite{li2024cot}.
% \end{fact}

\section{Main result}
\subsection{Task}
\label{sec:task}

\begin{definition}[Parity-Conditioned Retrieval]
\label{def:pmar}
Fix an integer $n\ge1$.  An instance consists of a bit string
\[
  Y=(Y_1,\ldots,Y_n)\in\zo^n
\]
and a query address $j\in[n]$.  Let
\[
  p(Y)=Y_1\oplus Y_2\oplus\cdots\oplus Y_n
\]
be the parity of the string.  The required output $\PCR_n(Y,j)$ is $\YES$ iff
\[
  Y_j\oplus p(Y)=1.
\]
and $\NO$ otherwise.
\end{definition}

For the formal separation we will encode the task in the following way.
\[
  \Sigma_{\rm in}=\{\BITZERO,\BITONE,\MARK,\BLANK\}.
\]
For a bit $a\in\zo$, write
\[
  B(a)=
  \begin{cases}
    (\BITZERO,\BITZERO), & a=0,\\
    (\BITONE,\BITONE), & a=1.
  \end{cases}
\]
For $j\in[n]$, let $Q(j)\in\{\MARK,\BLANK\}^n$ be the string with a single $\MARK$ at position $j$ and $\BLANK$ elsewhere.  The prompt for $(Y,j)$ is
\[
  \pi_n(Y,j)=B(Y_1)\cdots B(Y_n)\,Q(j),
  \qquad |\pi_n(Y,j)|=3n.
\]
\paragraph{Remark.} In the precise task formulation we are using each token of the $Y$ is encoded with two bits. This is a technical trick introduced so that the task can be solved by Qwen-style positive Gated DeltaNet. Alternatively, we could make our DeltaNets less restricted and allow for negative eigenvalues. This would be enough for the separation to hold with one-bit-per-token presentation.

\begin{theorem}[Fixed-precision scratchpad hierarchy]
\label{thm:main}
The Parity-Conditioned Retrieval defined above satisfies the following:
\begin{enumerate}[leftmargin=2em]
  \item For each n and N=3n, there exists a constant-precision hybrid decoder with one $\mathsf{GDN}$ layer and two $\mathsf{GA}$ layers that solves the \(\PCR_n\) in a single iteration (no scratchpad).  The model dimension $d_M$ is of order $O(\log N)$.
  \item Let $D_1,D_2\ldots$ be a family of constant-precision pure Gated Attention decoders with fixed depth and model dimension \(d_M=O(\log M)\) and  suppose that each $D_i$ solves the task on inputs of size $i$ with a scratch budget $S(N)$. then we have
  \[
    S(N)=N^{\Omega(1)}.
  \]
  \item Let $G$ be a constant-precision pure Gated DeltaNet decoder whose rounded recurrent state has $D(N)$ scalar coordinates over a fixed finite alphabet $Q$ on length-$N$ prompts.  If $G$ solves $\PCR_n$ on prompts $\pi_n(Y,j)$ of table length $n$, then
  \[
    D(N)\log_2|Q|\ge n-1.
  \]
  Consequently, every pure constant-precision Gated DeltaNet family of decoders with $D(N)\log_2|Q|=o(N)$ fails for the length family $N=3n$, even with unbounded scratchpad.
\end{enumerate}
\end{theorem}

The three parts are proved in Sections~\ref{sec:upper}, \ref{sec:ga-lower}, and~\ref{sec:gdn-lower}, respectively.

\section{Hybrid upper bound}
\label{sec:upper}

This section proves Item 1 of Theorem~\ref{thm:main}.  The proof has three components: a hard-selection attention gadget, a direct positive-gate Gated DeltaNet parity cell, and the assembly of these gadgets into a hybrid decoder for $\PCR$.

\subsection{Hard selection}

The attention gadget uses separated address codes rather than ordinary binary
strings.  For the fixed precision \(s\), choose a constant \(m_0(s)\) such that, for all
\(m\ge m_0(s)\),
\[
  \left[
    \exp\left(
      \left[-\frac{\sqrt{2m}}{3}\right]_s
    \right)
  \right]_s
  =
  0.
\]
For each \(n\), fix an error-correcting code
\[
  E_n:[n]\cup\{\MARK,\BLANK,\DUMMY\}\to\{-1,+1\}^{m_n},
\]
where \(m_n=O_s(\log n)\), \(m_n\ge m_0(s)\), and distinct codewords have
Hamming distance at least \(m_n/3\). Such codes exist by the standard
random-coding argument.

\begin{lemma}
\label{lem:equality}
\label{lem:equality}
Let \(\mathcal C\subseteq\{-1,+1\}^m\) have minimum Hamming distance at least
\(m/3\), with \(m\ge m_0(s)\). There is a single causal \(\mathsf{GA}\) head of
dimension \(2m\) which exactly selects the unique unmasked position whose code
equals the query code.
\end{lemma}

\begin{proof}
For \(c\in\mathcal C\), use
\[
  Q(c)=(c_1,1,c_2,1,\ldots,c_m,1),
  \qquad
  K(c)=(c_1,-1,c_2,-1,\ldots,c_m,-1).
\]
All coordinates are \(\pm1\), so RMS normalization leaves these vectors
unchanged. For \(c,d\in\mathcal C\),
\[
  \langle Q(c),K(d)\rangle_{\rm acc}
  =
  \sum_{r=1}^m(c_rd_r-1)
  =
  -2\Delta(c,d),
\]
where \(\Delta\) is Hamming distance. Hence the rounded QK-normalized score is
zero when \(c=d\). If \(c\ne d\), then
\[
  \frac{\langle Q(c),K(d)\rangle_{\rm acc}}{\sqrt{2m}}
  \le
  -\frac{\sqrt{2m}}{3},
\]
so by the choice of \(m_0(s)\),
\[
  \left[\exp(z(c,d))\right]_s=0.
\]
Meanwhile \(z(c,c)=0\), so
\[
  \left[\exp(z(c,c))\right]_s=1.
\]
Thus, if exactly one unmasked code equals the query code, the rounded softmax
has weight \(1\) on that position and weight \(0\) on all others. Since value
mixing uses rounded probabilities, the head returns exactly the value at the
matching position.
\end{proof}

\subsection{Gated DeltaNet parity cell}

\begin{lemma}[GDN parity solver]
\label{lem:gdn-toggle}
For every fixed precision \(s\ge2\), there is a two-coordinate
\(\mathsf{GDN}\) head that solves parity on the doubled bit encoding
\(B(0),B(1)\). 
\end{lemma}
\begin{proof}
The construction is more or less standard, the details can be encountered in the Appendix B.
\end{proof}

\subsection{Putting it all together}

\begin{theorem}[Hybrid upper bound]
\label{thm:hybrid-upper}
For each $N$ there exists a constant-precision hybrid decoder with one $\mathsf{GDN}$ layer and two $\mathsf{GA}$ layers which solves the Parity-Conditioned Retrieval in single iteration (no scratchpad).  The model dimension $d_M$ is of order $O(\log N)$.
\end{theorem}

\begin{proof}
Fix \(N=3n\), and let \(E_n\) be the address code from the hard-selection
subsection.  The embedding supplies token type, segment type, a value \(r\in\{1,2\}\) for the doubled table encoding, the local address code
\(E_n(\ell)\), and some \emph{protected residual coordinates}. All affine maps, output projections, and MLPs are chosen to act as identity on protected coordinates unless explicitly stated, and to output zero on irrelevant coordinates.

The \(\mathsf{GDN}\) layer computes the global parity.  It is initialized at
\(P_0\). For the first part of the input, a zero cell \(B(0)\) applies the two-token hold macro \((Z,Z)\), while a one
cell \(B(1)\) applies the two-token toggle macro \((G,F)\) from
Lemma~\ref{lem:gdn-toggle}. In the second part of the input (the query index segment) tokens apply the hold update \(Z\). Hence,
at the final prompt position, the recurrent state encodes \(p(Y)\), and the
readout from Lemma~\ref{lem:gdn-toggle} writes this bit into a protected
coordinate.

The first \(\mathsf{GA}\) layer retrieves the query address.  At the final
position, the query code is \(E_n(\MARK)\).  The unique marked token in the
query segment has key code \(E_n(\MARK)\), blank query tokens have key code
\(E_n(\BLANK)\), and table tokens have key code \(E_n(\DUMMY)\).  The value at
query offset \(\ell\) is \(E_n(\ell)\).  By Lemma~\ref{lem:equality}, the layer
writes \(E_n(j)\) exactly.

The second \(\mathsf{GA}\) layer retrieves the selected table bit.  Its query
code at the final position is the retrieved address \(E_n(j)\).  The first
token of table block \(\ell\) has key code \(E_n(\ell)\) and value \(Y_\ell\);
all other visible tokens have key code \(E_n(\DUMMY)\) and value zero.  Again by
Lemma~\ref{lem:equality}, this layer writes \(Y_j\) exactly.

A final rounded pointwise map computes \(Y_j\oplus p(Y)\) and the output layer
emits the corresponding answer token.  The GDN state has constant dimension and
the address code has length \(O(\log n)=O(\log N)\), so the model dimension is
\(O(\log N)\).
\end{proof}

\section{Pure Gated Attention lower bound}
\label{sec:ga-lower}

\begin{lemma}
\label{lem:ga-transcript-circuits}
Fix a constant-precision pure \(\mathsf{GA}\) decoder family of fixed depth, model dimension
\(d_M=O(\log M)\), and finite scratch alphabet \(\Gamma\).  There are constants
\(d_\star,c_\star\) such that, if the family recognizes a language $\mathcal{L}$ with a scratch budget \(S(N)\) in inputs of length $N$, then the accepted length-\(N\) instances
are recognized, under any fixed binary encoding of the finite vocabulary, by
circuits of depth \(d_\star\) and size
\[
  (N+S(N))^{c_\star}(|\Gamma|+1)^{S(N)+1}.
\]
\end{lemma}

The proof is deferred to Appendix C.

% \begin{lemma}[Transcript enumeration]
% \label{lem:transcripts}
% Let $D$ be a pure Gated Attention decoder with $d_0$ layers with scratch alphabet $\Gamma$.  If $D$ solves a promise problem on length-$N$ prompts with scratch budget $S(N)$, then the positive length-$N$ instances are recognized by $\ACzero$ circuits of depth at most $d_0+2$ and size
% \[
%   (N+S(N))^{O(1)}(|\Gamma|+1)^{S(N)+1}.
% \]
% \end{lemma}

% \begin{proof}
% Fix $N$ and write $S=S(N)$.  For a transcript $u=(u_1,\ldots,u_t)\in\Gamma^t$, where $0\le t\le S$, let $B_u(x)$ be the predicate that on prompt $x$ the decoder generates exactly $u_1,\ldots,u_t$ as scratch tokens and then emits $\YES$:
% \[
%   B_u(x)=
%   \left(\bigwedge_{r=1}^t [F_{N+r-1}(x,u_1,\ldots,u_{r-1})=u_r]\right)
%   \wedge [F_{N+t}(x,u_1,\ldots,u_t)=\YES].
% \]
% For fixed $u$, each next-token predicate has an $\ACzero$ circuit of depth at most $d_0$ and size at most $(N+S)^{c_0}$, with the transcript tokens hardwired.  Taking their conjunction adds one unbounded-fan-in AND layer and a polynomial size factor.  Thus $B_u$ has depth at most $d_0+1$ and size $(N+S)^{O(1)}$.

% The decoder accepts within its scratch budget iff
% \[
%   \bigvee_{t=0}^S\ \bigvee_{u\in\Gamma^t} B_u(x)
% \]
% holds.  The outer OR adds one layer and a factor at most
% \[
%   \sum_{t=0}^S |\Gamma|^t \le (|\Gamma|+1)^{S+1}.
% \]
% \end{proof}

\begin{theorem}[Pure Gated Attention lower bound]
\label{thm:ga-lower}
Every fixed-depth pure \(\mathsf{GA}\) decoder family with \(d_M=O(\log M)\)
and finite scratch alphabet that solves \(\PCR\) on length-\(N\) inputs while using
scratchpad of length \(S(N)\) satisfies, along \(N=3n\),
\[
  S(N)=N^{\Omega(1)}.
\]
\end{theorem}

\begin{proof}
Let \(N=3(r+1)\).  Given \(x=(x_1,\ldots,x_r)\in\{0,1\}^r\), set
\[
  Y_1=0,\qquad Y_{i+1}=x_i\quad(1\le i\le r),
\]
and choose query address \(j=1\).  Then
\[
  \PCR_{r+1}(Y,1)
  =
  Y_1\oplus(Y_1\oplus\cdots\oplus Y_{r+1})
  =
  x_1\oplus\cdots\oplus x_r.
\]
Thus parity on \(r\) bits is a length-linear projection of \(\PCR\), with
prompt length \(N=3(r+1)=\Theta(r)\).

If a pure \(\mathsf{GA}\) family solved \(\PCR\) with scratch budget \(S(N)\),
Lemma~\ref{lem:ga-transcript-circuits} would give depth-\(d_\star\) 
circuits for \(r\)-bit parity of size
\[
  (N+S(N))^{c_\star}(|\Gamma|+1)^{S(N)+1}.
\]
By H\aa stad's parity lower bound,
\[
  (N+S(N))^{c_\star}(|\Gamma|+1)^{S(N)+1}
  \ge
  2^{\Omega(r^{1/(d_\star-1)})}.
\]
Since \(r=\Theta(N)\), taking logarithms gives
\[
  c_\star\log(N+S(N))
  +(S(N)+1)\log(|\Gamma|+1)
  \ge
  \Omega(N^{1/(d_\star-1)}).
\]
If \(S(N)\) is already at least a fixed positive power of \(N\), we are done.
Otherwise \(\log(N+S(N))=O(\log N)\), which is lower order than the right-hand
side, and the inequality forces
\[
  S(N)=N^{\Omega(1)}.
\]
\end{proof}

\section{Pure Gated DeltaNet lower bound}
\label{sec:gdn-lower}

This section proves Item 3 of Theorem~\ref{thm:main}.  The argument only uses
the recurrent nature of pure Gated DeltaNet layers.  After a table prefix has
been read, all information about that prefix available to future tokens is
contained in the rounded recurrent states of the GDN heads.  

\begin{theorem}[Pure Gated DeltaNet  lower bound]
\label{thm:gdn-lower}
Let \(G\) be a pure \(\mathsf{GDN}\) decoder whose recurrent state after
reading $k$ tokens has \(D(k)\) scalar coordinates over a fixed finite
alphabet \(Q\).  If \(G\) solves \(\PCR_n\) on prompts \(\pi_n(Y,j)\), then
\[
  D(N)\log_2|Q|\ge n-1.
\]
\end{theorem}

\begin{proof}
For a table \(Y\in\{0,1\}^n\), let \(S(Y)\in Q^{D(N)}\) be the full rounded
recurrent state after the table segment
\[
  B(Y_1)\cdots B(Y_n)
\]
has been read, but before the query segment is read.  Since the decoder is pure
\(\mathsf{GDN}\), future behavior on any fixed query suffix is determined by
this state.

Define
\[
  R(Y)=(R_1(Y),\ldots,R_n(Y)),
  \qquad
  R_j(Y)=Y_j\oplus p(Y).
\]
The map \(R(Y)=Y\oplus p(Y)\mathbf 1^n\) is linear over \(\mathbb F_2\).  Its
kernel has size at most \(2\): it contains \(0^n\), and also \(1^n\) only when
\(n\) is odd.  Hence
\[
  |\operatorname{im} R|\ge 2^{n-1}.
\]

If two tables \(Y,Y'\) satisfy \(S(Y)=S(Y')\), then appending the same query
segment \(Q(j)\) gives identical future computations.  Any generated scratch
tokens are deterministic functions of the same recurrent state and same suffix,
so the continuations remain identical until the first answer token.  Correctness
for every \(j\) therefore implies
\[
  R_j(Y)=R_j(Y')\quad\text{for all }j,
\]
and hence \(R(Y)=R(Y')\).

Thus the recurrent state must distinguish at least \(2^{n-1}\) different
response vectors.  But it has at most \(|Q|^{D(N)}\) possible values.  Therefore
\[
  |Q|^{D(N)}\ge 2^{n-1},
\]
which gives
\[
  D(N)\log_2|Q|\ge n-1.
\]
\end{proof}

\begin{proof}[Proof of Theorem~\ref{thm:main}]
Item 1 is Theorem~\ref{thm:hybrid-upper}.  Item 2 is Theorem~\ref{thm:ga-lower}. Item 3 is Theorem~\ref{thm:gdn-lower}.
\end{proof}

\newpage
\section*{Appendix}
\subsection*{Appendix A: Constant precision}
\begin{definition}[Constant precision]
Fix an integer \(s\ge2\).  Let
\[
  \mathbb F_s=\{0\}\cup\{\pm k2^{-s}:1\le k\le 2^{2s}-1\}.
\]
Let \(B_s=2^s-2^{-s}\) be the largest positive element of \(\mathbb F_s\).
The rounding map \([\cdot]_s:\mathbb R\to\mathbb F_s\) sends a real number to
the closest element of \(\mathbb F_s\), breaking ties by smaller absolute value
and saturating to \(\pm B_s\) on overflow.  Vectors and matrices are rounded
coordinatewise.
\end{definition}

For \(a,b\in\mathbb F_s\), define rounded scalar arithmetic by
\[
  a\oplus_s b=[a+b]_s,\qquad
  a\ominus_s b=[a-b]_s,\qquad
  a\otimes_s b=[ab]_s.
\]
For \(x=(x_1,\ldots,x_m)\in\mathbb F_s^m\), define the strictly rounded sum
\[
  \operatorname{sum}_s(x_1,\ldots,x_m)
  =(((x_1\oplus_s x_2)\oplus_s x_3)\cdots)\oplus_s x_m,
\]
with the evident convention for \(m=1\).  The corresponding strictly rounded
inner product is
\[
  \langle x,y\rangle_s
  =
  \operatorname{sum}_s(x_1\otimes_s y_1,\ldots,x_m\otimes_s y_m).
\]
This strictly rounded arithmetic is the default convention outside the
designated reductions below.  In particular, the Gated DeltaNet row update used
in Section~\ref{sec:upper} uses \(\langle\cdot,\cdot\rangle_s\).

\paragraph{Temporary reduction accumulators.}
For designated reductions, we write
\[
  \operatorname{sum}_{\rm acc}(x_1,\ldots,x_m)=\sum_{r=1}^m x_r
\]
for a temporary accumulator sum.  Equivalently, the accumulator has enough
range and precision to determine the same rounded result as the exact real
sum.  Since the accumulated operands are elements of the fixed finite set
\(\mathbb F_s\), an accumulator with \(O(\log m)\) integer bits suffices for
the ranges used below.  The accumulator value itself is discarded after the
operator returns its rounded output.

For \(x,y\in\mathbb F_s^m\), define the accumulator dot product
\[
  \langle x,y\rangle_{\rm acc}
  =
  \sum_{r=1}^m x_ry_r.
\]
The QK-normalized finite-precision score is
\[
  \operatorname{score}_s(x,y)
  =
  \left[
    \frac{\langle x,y\rangle_{\rm acc}}{\sqrt m}
  \right]_s.
\]
Thus the dot product may be accumulated before rounding, but the score consumed
by the exponential/softmax is again an element of \(\mathbb F_s\).

For RMS normalization, define
\[
  \rho(x)
  =
  \sqrt{\frac1m\operatorname{sum}_{\rm acc}(x_1^2,\ldots,x_m^2)}.
\]
If \(\rho(x)=0\), set \(\operatorname{RMSNorm}_s(x)=0\).  Otherwise set
\[
  \operatorname{RMSNorm}_s(x)_r
  =
  \left[\frac{x_r}{\rho(x)}\right]_s,
  \qquad r=1,\ldots,m.
\]
Again, the reduction computing \(\rho(x)\) may use a temporary accumulator, but
the normalized vector is rounded coordinatewise to \(\mathbb F_s^m\).

The \(\ell_2\)-normalization used for GDN keys and queries is defined
analogously.  Let
\[
  \rho_2(x)=\sqrt{\operatorname{sum}_{\rm acc}(x_1^2,\ldots,x_m^2)}.
\]
If \(\rho_2(x)=0\), the normalized vector is \(0\).  Otherwise its \(r\)-th
coordinate is
\[
  \left[\frac{x_r}{\rho_2(x)}\right]_s .
\]

Finally, for \(z=(z_1,\ldots,z_m)\in\mathbb F_s^m\), define rounded softmax by
\[
  e_r=[\exp(z_r)]_s\in\mathbb F_s,\qquad r=1,\ldots,m,
\]
\[
  D=\operatorname{sum}_{\rm acc}(e_1,\ldots,e_m).
\]
If \(D=0\), set \(\operatorname{softmax}_s(z)=0\).  If \(D>0\), set
\[
  \operatorname{softmax}_s(z)_r
  =
  \left[\frac{e_r}{D}\right]_s.
\]
The softmax denominator may therefore be accumulated in a wider temporary type,
but each attention probability is rounded back to \(\mathbb F_s\) before value
mixing.

In attention layers, value mixing uses these rounded probabilities:
\[
  u_i^a
  =
  \operatorname{sum}_s
  \bigl(
    A_{i1}^a\otimes_s v_1^a,\ldots,A_{ii}^a\otimes_s v_i^a
  \bigr),
\]
coordinate-wise.  
% Thus the model does not have access to an unrounded fused
% quantity of the form
% \[
%   \frac{\sum_j \exp(s_{ij})v_j}{\sum_j \exp(s_{ij})}.
% \]

\subsection*{Appendix B: Proof of Lemma~\ref{lem:gdn-toggle}}

We give the construction for the fixed precision \(s\ge2\).  Let
\[
  \delta=2^{-s},
  \qquad
  \kappa=\left[1/\sqrt2\right]_s .
\]
Then \(1/2<\kappa\le 3/4\).  Hence
\[
  [\kappa\delta]_s=\delta,\qquad
  [2\kappa\delta]_s=\delta,\qquad
  [3\kappa\delta]_s=2\delta,
\]
and
\[
  [(1/2)\delta]_s=0,\qquad
  [(3/4)\delta]_s=\delta,
\]
where ties are rounded toward smaller absolute value.

We use one \(\mathsf{GDN}\) head with \(d_h=2\).  Only the first row of the
state is active:
\[
  S_i=
  \begin{pmatrix}
    x_i\\
    0\;\;0
  \end{pmatrix},
  \qquad x_i\in\mathbb F_s^2 .
\]
The two parity states are
\[
  P_0=(0,\delta),
  \qquad
  P_1=(\delta,0),
\]
and the initial state is \(P_0\).

For one active row, the GDN update has the form
\[
  U_{\alpha,\beta,k,v}(x)
  =
  \alpha\left(x-\beta\langle x,k\rangle_s k\right)+\beta vk,
\]
with all arithmetic interpreted under the convention of Appendix~A.  We write
\(v\) for the first coordinate of the value vector; the second coordinate is
zero, so only the first row of the state is active.

We use three updates: a hold update \(Z\), and two phase updates \(G,F\).

The hold update is
\[
  \alpha_Z=3/4,\qquad
  \beta_Z=1/4,\qquad
  k_Z=(1,0),\qquad
  v_Z=0.
\]
It fixes both parity states.  Indeed, for \(P_0=(0,\delta)\), the inner product
with \(k_Z\) is \(0\), and \((3/4)\delta\) rounds to \(\delta\).  For
\(P_1=(\delta,0)\),
\[
  \beta_Z\langle P_1,k_Z\rangle_s=(1/4)\delta\mapsto 0,
\]
and again \((3/4)\delta\mapsto\delta\).  Hence
\[
  Z(P_0)=P_0,\qquad Z(P_1)=P_1.
\]

The first phase update is
\[
  \alpha_G=3/4,\qquad
  \beta_G=1/2,\qquad
  k_G=(\kappa,\kappa),\qquad
  v_G=6\delta.
\]
For both \(P_0\) and \(P_1\),
\[
  \langle P_i,k_G\rangle_s=\delta,
  \qquad
  \beta_G\langle P_i,k_G\rangle_s=(1/2)\delta\mapsto 0.
\]
Thus the rank-one correction vanishes on both parity states.  Also
\[
  \beta_G v_G=(1/2)(6\delta)=3\delta,
\]
and
\[
  3\delta k_G=(2\delta,2\delta),
\]
using \([3\kappa\delta]_s=2\delta\).  Therefore
\[
  G(P_0)=(2\delta,3\delta)=:H_0,
  \qquad
  G(P_1)=(3\delta,2\delta)=:H_1.
\]

The second phase update is
\[
  \alpha_F=1/4,\qquad
  \beta_F=3/4,\qquad
  k_F=(-\kappa,\kappa),\qquad
  v_F=0.
\]
For \(H_0=(2\delta,3\delta)\),
\[
  \langle H_0,k_F\rangle_s
  =
  -[2\kappa\delta]_s+[3\kappa\delta]_s
  =
  -\delta+2\delta
  =
  \delta.
\]
Hence
\[
  \beta_F\langle H_0,k_F\rangle_s=(3/4)\delta\mapsto\delta,
\]
and
\[
  H_0-\delta k_F
  =
  (2\delta,3\delta)-(-\delta,\delta)
  =
  (3\delta,2\delta).
\]
Multiplying by \(\alpha_F=1/4\) gives
\[
  F(H_0)=(\delta,0)=P_1.
\]

Similarly, for \(H_1=(3\delta,2\delta)\),
\[
  \langle H_1,k_F\rangle_s
  =
  -[3\kappa\delta]_s+[2\kappa\delta]_s
  =
  -2\delta+\delta
  =
  -\delta.
\]
Thus
\[
  \beta_F\langle H_1,k_F\rangle_s\mapsto -\delta,
\]
and
\[
  H_1-(-\delta)k_F
  =
  (3\delta,2\delta)-(\delta,-\delta)
  =
  (2\delta,3\delta).
\]
Multiplying by \(\alpha_F=1/4\) gives
\[
  F(H_1)=(0,\delta)=P_0.
\]

Thus, the two-token macro-update \((G,F)\) swaps \(P_0\) and \(P_1\), while
\((Z,Z)\) preserves both states.  The doubled bit encoding is implemented by
\[
  B(0)\mapsto (Z,Z),
  \qquad
  B(1)\mapsto (G,F).
\]
Query-segment tokens also use the hold update \(Z\), so the parity state is
preserved after the table segment.

Finally, read the state using the normalized query
\[
  q_{\rm par}=(1,0).
\]
Then \(S_iq_{\rm par}=0\) in state \(P_0\), while
\[
  S_iq_{\rm par}=(\delta,0)
\]
in state \(P_1\).  After RMS normalization, zero remains zero and the nonzero
read vector remains nonzero, so the two states are distinguishable.  The output
gate is chosen to have a strictly positive post-rounded value on the active
readout coordinate, and a fixed position-wise MLP writes the parity bit into a
protected residual coordinate.

\subsection*{Appendix C: Proof of Lemma~\ref{lem:ga-transcript-circuits}}
\label{app:ga-transcript-circuits}

\begin{proof}
Fix a context length \(M\).  We first show that, for every token
\(a\in\Sigma\), the one-step predicate
\[
  \operatorname{Next}_M(z_1,\ldots,z_M)=a
\]
is computed by circuits of constant depth and size
\(M^{O(1)}\).

The fact that \(\Sigma\) and \(\Gamma\) need not be binary causes no difficulty.
Since the vocabulary is finite, fix a binary encoding of \(\Sigma\) using
\[
  b_\Sigma=\lceil\log_2|\Sigma|\rceil=O(1)
\]
bits per token.  A predicate such as ``the next token is \(a\)'' is therefore a
Boolean predicate over \(O(1)\) output bits.  For example, a four-symbol scratch
alphabet can be represented using two bits.  Different fixed encodings change
only constant factors in depth and size.

We prove the one-step simulation by induction over layers.  For every layer
\(\ell\), position \(i\), coordinate \(r\), and rounded value
\(b\in\mathbb F_s\), we construct a constant-depth polynomial-size circuit for
the predicate
\[
  h_i^\ell[r]=b.
\]
The base case is immediate: the embedding is a finite-valued function of the
token at position \(i\) and the hardwired position \(i\).

Pointwise rounded operations are finite lookups on \(O(d_M)=O(\log M)\) rounded
inputs.  This covers affine maps, gates, position-wise MLPs, RMS-normalized
coordinates, and QK scores.  Even when an RMS reduction or a QK dot product uses
a temporary accumulator, the value exposed to later computation is rounded back
to the fixed finite set \(\mathbb F_s\).  Hence each output-value predicate can
be represented by a DNF over all assignments to \(O(\log M)\) finite-valued
inputs, of size
\[
  |\mathbb F_s|^{O(\log M)}=M^{O(1)}.
\]

It remains to handle the softmax over the causal prefix.  Let
\[
  e_{ij}=[\exp(s_{ij})]_s.
\]
The values \(e_{ij}\) lie in a fixed finite subset of
\(\mathbb F_s^{\ge0}\).  Let \(\delta_s>0\) be the smallest positive element of
\(\mathbb F_s\), and let \(\theta_s>0\) be the positive rounding threshold below
which a nonnegative real rounds to zero.  If the rounded attention probability
\[
  A_{ij}
  =
  \left[
    \frac{e_{ij}}{\sum_{r\le i}e_{ir}}
  \right]_s
\]
is nonzero, then
\[
  \frac{e_{ij}}{\sum_{r\le i}e_{ir}}>\theta_s.
\]
Since \(e_{ij}\le B_s\), the softmax denominator is \(O_s(1)\).  Because every
positive \(e_{ir}\) is at least \(\delta_s\), there is a constant \(K_s\),
depending only on the precision, such that whenever any rounded attention
probability is nonzero, at most \(K_s\) positions have \(e_{ir}>0\).

Thus each attention output coordinate can be computed by enumerating constant
size active sets.  For every set \(T\subseteq[i]\) with \(|T|\le K_s\), and
every assignment of positive rounded exponential values
\(\eta_r\in\mathbb F_s^{>0}\) for \(r\in T\), consider the event
\[
  \bigwedge_{r\in T}(e_{ir}=\eta_r)
  \wedge
  \bigwedge_{r\le i,\ r\notin T}(e_{ir}=0).
\]
There are only \(M^{O(K_s)}=M^{O(1)}\) such cases.  On each case, the softmax
denominator is fixed, every rounded attention probability is a fixed element of
\(\mathbb F_s\), and the value mixture depends on only \(O_s(1)\) value
coordinates.  Therefore the corresponding output predicate is a constant-size
finite lookup over already constructed predicates.

If more than \(K_s\) exponentials are positive, then all rounded attention
probabilities are zero.  This case is also recognized in \(\ACzero\), by an OR
over \((K_s+1)\)-tuples of positions with positive exponentials.  Hence
attention-output predicates have constant-depth polynomial-size circuits.

Composing over the fixed number of layers preserves constant depth and
polynomial size.  The final rounded logits and deterministic tie-broken argmax
are finite lookups, so the one-step predicate
\[
  \operatorname{Next}_M(z_1,\ldots,z_M)=a
\]
has constant-depth polynomial-size Boolean circuits.

We now enumerate all possible scratch completions.  Fix a length-\(N\) input and a candidate
scratch completion
\[
  \tau=(\gamma_1,\ldots,\gamma_t),
  \qquad t\le S(N),
\]
where each \(\gamma_r\in\Gamma\).  The symbols of \(\tau\) are hardwired
constants in the circuit.  The completion is consistent and accepting exactly
when
\[
  \operatorname{Next}_{N+r-1}
  (x,\gamma_1,\ldots,\gamma_{r-1})
  =
  \gamma_r
  \qquad 1\le r\le t,
\]
and finally
\[
  \operatorname{Next}_{N+t}
  (x,\gamma_1,\ldots,\gamma_t)
  =
  \YES.
\]
Each of these \(t+1\) checks has constant depth and size polynomial in
\(N+S(N)\).  Their conjunction adds one unbounded-fan-in AND layer and only a
polynomial size factor.

The decoder accepts within its scratch budget iff at least one accepting
transcript exists.  Therefore we OR over all transcripts of length at most
\(S(N)\).  The number of such transcripts is
\[
  \sum_{t=0}^{S(N)}|\Gamma|^t
  \le
  (|\Gamma|+1)^{S(N)+1}.
\]
The outer OR adds one more layer and the displayed multiplicative size factor.
Absorbing polynomial factors into constants gives depth \(d_\star\) and size
\[
  (N+S(N))^{c_\star}(|\Gamma|+1)^{S(N)+1}.
\]
\end{proof}

% \newpage
% \input{checklist.tex}


\begin{thebibliography}{99}

\bibitem{angluin2023masked}
Andy Yang, David Chiang, and Dana Angluin.
\newblock Masked hard-attention transformers recognize exactly the star-free languages.
\newblock In \emph{Advances in Neural Information Processing Systems}, 2024.

\bibitem{barcelo2023logical}
Pablo Barcel\'o, Alexander Kozachinskiy, Anthony Widjaja Lin, and Vladimir V. Podolskii.
\newblock Logical languages accepted by transformer encoders with hard attention.
\newblock In \emph{International Conference on Learning Representations}, 2024.

\bibitem{grazzi2025negative}
Riccardo Grazzi, Julien Siems, J\"org K. H. Franke, Arber Zela, Frank Hutter, and Massimiliano Pontil.
\newblock Unlocking state-tracking in linear RNNs through negative eigenvalues.
\newblock In \emph{International Conference on Learning Representations}, 2025.

\bibitem{hahn2020theoretical}
Michael Hahn.
\newblock Theoretical limitations of self-attention in neural sequence models.
\newblock \emph{Transactions of the Association for Computational Linguistics}, 8:156--171, 2020.

\bibitem{hao2022formal}
Yiding Hao, Dana Angluin, and Robert Frank.
\newblock Formal language recognition by hard attention transformers: Perspectives from circuit complexity.
\newblock \emph{Transactions of the Association for Computational Linguistics}, 10:800--810, 2022.

\bibitem{jiang2025softmax_tc}
Hongjian Jiang, Michael Hahn, Georg Zetzsche, and Anthony Widjaja Lin.
\newblock Softmax Transformers are Turing-Complete.
\newblock In \emph{International Conference on Learning Representations}, 2026.

\bibitem{kozachinskiy2026parity}
Alexander Kozachinskiy, Tomasz Steifer, and Przemys{\l}aw Wa{\l}\k{e}ga.
\newblock Parity, sensitivity, and transformers.
\newblock \emph{arXiv preprint arXiv:2602.05896}, 2026.

\bibitem{merrill2022saturated}
William Merrill, Ashish Sabharwal, and Noah A. Smith.
\newblock Saturated transformers are constant-depth threshold circuits.
\newblock \emph{Transactions of the Association for Computational Linguistics}, 10:843--856, 2022.

\bibitem{merrill2023parallelism}
William Merrill and Ashish Sabharwal.
\newblock The parallelism tradeoff: Limitations of log-precision transformers.
\newblock \emph{Transactions of the Association for Computational Linguistics}, 11:531--545, 2023.

\bibitem{merrill2024cot}
William Merrill and Ashish Sabharwal.
\newblock The expressive power of transformers with chain of thought.
\newblock In \emph{International Conference on Learning Representations}, 2024.

\bibitem{merrill2026olmo_hybrid}
William Merrill, Yanhong Li, Tyler Romero, Anej Svete, Caia Costello, Pradeep Dasigi, Dirk Groeneveld, David Heineman, Bailey Kuehl, Nathan Lambert, Chuan Li, Kyle Lo, Saumya Malik, DJ Matusz, Benjamin Minixhofer, Jacob Morrison, Luca Soldaini, Finbarr Timbers, Pete Walsh, Noah A. Smith, Hannaneh Hajishirzi, and Ashish Sabharwal.
\newblock Olmo Hybrid: From theory to practice and back.
\newblock \emph{arXiv preprint arXiv:2604.03444}, 2026.

\bibitem{qwen3next_code}
Qwen Team and Hugging Face.
\newblock PyTorch implementation of Qwen3-Next, \texttt{modular\_qwen3\_next.py}.
\newblock Hugging Face Transformers, 2025.
\newblock \url{https://github.com/huggingface/transformers/blob/main/src/transformers/models/qwen3_next/modular_qwen3_next.py}.

\bibitem{siegelmann1995computational}
Hava T. Siegelmann and Eduardo D. Sontag.
\newblock On the computational power of neural nets.
\newblock \emph{Journal of Computer and System Sciences}, 50(1):132--150, 1995.

\bibitem{siegelmann1996finite}
Hava T. Siegelmann.
\newblock Recurrent neural networks and finite automata.
\newblock \emph{Computational Intelligence}, 12(4):567--574, 1996.

\bibitem{siems2025deltaproduct}
Julien Siems, Timur Carstensen, Arber Zela, Frank Hutter, Massimiliano Pontil, and Riccardo Grazzi.
\newblock DeltaProduct: Improving state-tracking in linear RNNs via Householder products.
\newblock \emph{arXiv preprint arXiv:2502.10297}, 2025.

\bibitem{strobl2024survey}
Lena Strobl, William Merrill, Gail Weiss, David Chiang, and Dana Angluin.
\newblock What formal languages can transformers express? A survey.
\newblock \emph{Transactions of the Association for Computational Linguistics}, 12:543--561, 2024.

\bibitem{svete2024lower}
Anej Svete, Franz Nowak, Anisha Mohamed Sahabdeen, and Ryan Cotterell.
\newblock Lower bounds on the expressivity of recurrent neural language models.
\newblock In \emph{Proceedings of the 2024 Conference of the North American Chapter of the Association for Computational Linguistics}, pages 6820--6844, 2024.

\bibitem{yang2024counting}
Andy Yang and David Chiang.
\newblock Counting like transformers: Compiling temporal counting logic into softmax transformers.
\newblock \emph{arXiv preprint arXiv:2404.04393}, 2024.

\bibitem{furst1984}
Merrick Furst, James B. Saxe, and Michael Sipser.
\newblock Parity, circuits, and the polynomial-time hierarchy.
\newblock \emph{Mathematical Systems Theory}, 17:13--27, 1984.

\bibitem{hastad1986}
Johan H\aa stad.
\newblock Almost optimal lower bounds for small depth circuits.
\newblock In \emph{Proceedings of the 18th Annual ACM Symposium on Theory of Computing}, pages 6--20, 1986.

\bibitem{li2024cot}
Zhiyuan Li, Hong Liu, Denny Zhou, and Tengyu Ma.
\newblock Chain of thought empowers transformers to solve inherently serial problems.
\newblock In \emph{International Conference on Learning Representations}, 2024.

\bibitem{qwen3next_modelcard}
Qwen Team.
\newblock Qwen3-Next-80B-A3B-Instruct model card.
\newblock Hugging Face, 2025.
\newblock \url{https://huggingface.co/Qwen/Qwen3-Next-80B-A3B-Instruct}.

\bibitem{yang2025gdn}
Songlin Yang, Jan Kautz, and Ali Hatamizadeh.
\newblock Gated Delta Networks: Improving Mamba2 with delta rule.
\newblock In \emph{International Conference on Learning Representations}, 2025.

\end{thebibliography}
\end{document}